\documentclass[letterpaper]{article} % DO NOT CHANGE THIS
\usepackage{aaai25}  % DO NOT CHANGE THIS
\usepackage{times}  % DO NOT CHANGE THIS
\usepackage{helvet}  % DO NOT CHANGE THIS
\usepackage{courier}  % DO NOT CHANGE THIS
\usepackage[hyphens]{url}  % DO NOT CHANGE THIS
\usepackage{graphicx} % DO NOT CHANGE THIS
\usepackage{natbib}  % DO NOT CHANGE THIS AND DO NOT ADD ANY OPTIONS TO IT
\usepackage{caption} % DO NOT CHANGE THIS AND DO NOT ADD ANY OPTIONS TO IT
\usepackage{algorithm}
\usepackage{amsmath}
\usepackage{algpseudocode}
\usepackage{xspace}
\usepackage{todonotes}
\usepackage{subcaption}
\usepackage{graphics}
\usepackage{comment}
\usepackage{booktabs}
\usepackage{multicol,multirow}
\usepackage{amsthm}
\usepackage{amsfonts}
\usepackage{siunitx}
\usepackage{newfloat}
\usepackage{listings}
\usepackage{thm-restate}
\usepackage{paralist}
\newcommand{\D}{\mathcal{D}}

\newcommand{\limp}{\mathbin{\rightarrow}}

\RequirePackage{xparse}
\RequirePackage{xspace}
\RequirePackage{mathtools}

\NewDocumentCommand\until       {}{\mathbin{\mathsf{U}}}

\NewDocumentCommand\tomorrow    {}{\mathsf{X}}

\NewDocumentCommand\eventually  {}{\mathsf{F}}

\DeclareMathOperator\wnextvar{\bigcirc\kern-1.17em\sim}

\newcommand{\true}{\mathit{true}}

\newcommand{\false}{\mathit{false}}

\newcommand{\length}{\mathit{len}}

\newcommand{\LTL}{{\sc ltl}\xspace}
\newcommand{\LTLf}{{\sc ltl}\ensuremath{_f}\xspace}
\newcommand{\LTLp}{{\sc ltl}\ensuremath{_p}\xspace}

\newcommand{\DFA}{{\sc dfa}\xspace}

\newcommand{\Nat}{{\rm I\kern-.23em N}}

\newcommand{\tup}[1]{\langle #1 \rangle}

\newcommand{\tasks}{\ensuremath{\Sigma}\xspace}
\newcommand{\atask}{a}
\newcommand{\task}[1]{\textnormal{\textsc{#1}}}
\newcommand{\trace}{\ensuremath{\tau}\xspace}
\newcommand{\activetasks}[1]{\tasks_{#1}\xspace}
\newcommand{\othertasks}[1]{\overline{\activetasks{#1}}}

\newcommand{\temptrue}{\mathit{temp\_true}}
\newcommand{\tempfalse}{\mathit{temp\_false}}
\newcommand{\permtrue}{\mathit{perm\_true}}
\newcommand{\permfalse}{\mathit{perm\_false}}

\newcommand{\aut}{A}
\newcommand{\states}{Q}
\newcommand{\state}{q}
\newcommand{\istate}{\state_0}
\newcommand{\fstates}{F}
\newcommand{\trans}{\delta}

\newcommand{\statesize}{.5 cm}

\tikzstyle{dot}=[
    circle,
    ultra thick,
    minimum width=1mm,
    minimum height=1mm,
    draw,
  ]

\tikzstyle{state}=[circle,draw=black,very thick,minimum width=\statesize]

\tikzstyle{istate}=[state,initial,initial text={\!\!\!\!\!},initial distance=3mm]

\tikzstyle{truecolor}=[fill=green!30]
\tikzstyle{temptruecolor}=[fill=blue!30]
\tikzstyle{falsecolor}=[fill=red!30]
\tikzstyle{tempfalsecolor}=[fill=orange!30]

\tikzstyle{truestate}=[state,truecolor,line width=2pt]
\tikzstyle{temptruestate}=[state,temptruecolor,very thick]
\tikzstyle{falsestate}=[state,falsecolor,very thick,densely dotted]
\tikzstyle{tempfalsestate}=[state,tempfalsecolor,very thick,densely dashed]

\tikzstyle{monitorstate}=[
          very thick,
          draw,
          inner sep=0pt]

\tikzstyle{truemonstate}=[
          monitorstate,
          truecolor,
          label=center:$\permtrue$]
          
\tikzstyle{temptruemonstate}=[
          monitorstate,
          very thick,
          temptruecolor,
          label=center:$\temptrue$]
          
\tikzstyle{falsemonstate}=[
          monitorstate,
          very thick,
          falsecolor,
          label=center:$\permfalse$]
                    
\tikzstyle{tempfalsemonstate}=[
          monitorstate,
          very thick,
          tempfalsecolor,
          label=center:$\tempfalse$]

\newcommand{\set}[1]{\{ #1 \}}

\usetikzlibrary{matrix,automata,arrows,calc,fit}
\usetikzlibrary{decorations.pathreplacing,positioning}

\DeclareCaptionStyle{ruled}{labelfont=normalfont,labelsep=colon,strut=off} % DO NOT CHANGE THIS
\floatstyle{ruled}
\newfloat{listing}{tb}{lst}{}
\floatname{listing}{Listing}
\title{Generating Counterfactual Explanations Under Temporal Constraints}
\author{
    Andrei Buliga\textsuperscript{\rm 1,2}\equalcontrib, Chiara Di Francescomarino\textsuperscript{\rm 3}\equalcontrib,
    Chiara Ghidini\textsuperscript{\rm 2}\equalcontrib,
    Marco Montali\textsuperscript{\rm 2}\equalcontrib,
    Massimiliano Ronzani\textsuperscript{\rm 1}\equalcontrib}
\affiliations{
    \textsuperscript{\rm 1}Fondazione Bruno Kessler, 
    Via Sommarive, 18, POVO - 38123, Trento, Italy\\
    \textsuperscript{\rm 2}Free University of Bozen-Bolzano, 
    via Bruno Buozzi, 1 - 39100, Bozen-Bolzano, Italy\\
    \textsuperscript{\rm 3}University of Trento, 
    Via Sommarive, 9, 38123 Trento, Italy\\
    \{abuliga, mronzani\}@fbk.eu, c.difrancescomarino@unitn.it, \{chiara.ghidini, marco.montali\}@unibz.it
}

\newcommand{\dist}{\emph{distance}\xspace}
\newcommand{\spars}{\emph{sparsity}\xspace}
\newcommand{\impl}{\emph{implausibility}\xspace}
\newcommand{\dive}{\emph{diversity}\xspace}

\newcommand{\run}{\emph{runtime}\xspace}
\newcommand{\sat}{\emph{compliance}\xspace}
\newcommand{\firstheur}{{{\texttt{aPriori}}}\xspace}
\newcommand{\secondheur}{{{\texttt{Online}}}\xspace}
\newcommand{\mar}{\emph{Mutate-And-Retry}\xspace}

\newcommand{\ltlfbase}{\emph{$\text{Genetic}_\varphi$}\xspace}
\newcommand{\marshort}{\texttt{MAR}\xspace}

\newcommand{\ltlfbaseshort}{$\text{\texttt{Gen}}_\varphi$\xspace}
\newcommand{\massi}[1]{\textcolor{black}{#1}}

\newcommand{\btext}[1]{\textcolor{black}{#1}}
\newcommand{\submitapplication}{\task{apply}\xspace}
\newcommand{\automaticcheck}{\task{aut-chk}\xspace}
\newcommand{\manualcheck}{\task{man-chk}\xspace}
\newcommand{\informphone}{\task{phone}\xspace}
\newcommand{\informsms}{\task{sms}\xspace}
\newcommand{\informemail}{\task{email}\xspace}
\newcommand{\askappointment}{\task{book}\xspace}
\newcommand{\createoffer}{\task{offer}\xspace}
\newcommand{\senddocument}{\task{send-doc}\xspace}
\newcommand{\accept}{\task{ok}\xspace}
\newcommand{\wrongtask}[1]{\underline{#1}}

\newcommand{\formulacheck}{\varphi_{\mathit{chk}}}
\newcommand{\formulacomm}{\varphi_{\mathit{comm}}}
\newcommand{\rank}[1]{\ \hfill  \mathsf{(#1)}}

\newcommand{\winner}[1]{\mathbf{#1}}

\begin{document}

\maketitle

\begin{abstract}
Counterfactual explanations are one of the prominent eXplainable Artificial Intelligence (XAI) techniques, and suggest changes to input
data that could alter predictions, leading to more favourable outcomes.
Existing counterfactual methods do not readily apply to temporal domains, such as that of process mining, where data take the form of traces of activities that must obey to temporal background knowledge expressing which dynamics are possible and which not.  
%Existing methods fail to guarantee that
Specifically, counterfactuals generated off-the-shelf may violate the background knowledge, leading to inconsistent explanations.
%This gap is especially crucial in these aforementioned temporally dependent domains, where the sequence and ordering of events within the data plays a key role in accurate predictions and actionable explanations. %(questa forse si puo saltare)
This work tackles this challenge by introducing a novel approach for generating \emph{temporally constrained counterfactuals}, guaranteed to comply
by design with background knowledge expressed in Linear Temporal Logic on
process traces (\LTLp).
We do so by infusing automata-theoretic techniques for \LTLp inside a genetic algorithm for counterfactual generation.
The empirical evaluation shows that the generated counterfactuals are temporally meaningful and more interpretable for applications involving temporal dependencies.

\end{abstract}

\begin{links}
\link{Code}{https://github.com/abuliga/AAAI2025-temporal-constrained-counterfactuals}
\label{linkcode}
\end{links}
%%%%%%%%%%%%%%%%%%%%%
%!TEX root = ./main.tex

\section{Introduction}
\label{sec:intro}

\btext{State-of-the-art Machine Learning efforts prioritise accuracy using ensemble and deep learning techniques, but their complexity makes their input-output mappings difficult to interpret. To address this, eXplainable Artificial Intelligence (XAI) techniques have emerged, which aid in the interpretation of predictions and promoting the adoption of advanced models~\cite{verma2020counterfactual,guidotti2022counterfactual,moc,priorxai}.}

Counterfactual explanations~\cite{verma2020counterfactual,guidotti2022counterfactual} are a key eXplainable Artificial Intelligence (XAI) technique. They provide insights into which changes should be applied to an input instance to alter the outcome of a prediction. Such explanations are hence particularly valuable for users who need to understand how different attributes or actions might influence an outcome of interest. State-of-the-art counterfactual generation methods often rely on optimisation techniques to find minimal changes to inputs leading to altering the predicted outcome. 

Existing methods do not readily apply to temporal domains, such as that of process mining \cite{DBLP:books/sp/Aalst16}, where data of interest consists of traces of activities generated by executing a business/work process or a plan. Such so-called \emph{process traces}~\cite{FioG18} are increasingly used in key domains like healthcare, business, and industrial processes, where the sequencing of activities is central. The main issue in these settings is that not all sequencings of activities make sense: traces are typically subject to temporal constraints, that is, must comply with temporal background knowledge expressing which dynamics are possible and which not. E.g., in a healthcare process a patient may enter triage only upon giving privacy consent. 
% account for temporal data temporal dependencies in data. These temporal dependencies are crucial in sequences of activities that occur over time, known as \emph{process traces}~\cite{FioG18}, that is, data represented by a non-empty, finite sequence $\trace = \atask_1,\ldots,\atask_{n}$ of activities $\atask_i$ executed at time $i$.

Systematic solutions for generating counterfactual explanations that comply with temporal background knowledge are still lacking~\cite{caisepaper}. This is a problem in areas such as Predictive Process Monitoring (PPM)~\cite{DBLP:books/sp/22/FrancescomarinoG22}, a widely established framework in process mining. Here, %Machine Learning (ML) based
\btext{black-box} predictive models are typically used to predict %, among the others,
the outcome of executions of business and work processes, often making interpretability a challenge. Integrating counterfactual explanations into such monitoring frameworks can enhance understanding by providing alternative trace executions to reach a more favourable outcome, only upon guaranteeing that such explanations make sense, that is, comply with temporal background knowledge. 
%However, generating counterfactuals that respect the temporal nature of process traces and the actual domain of the process itself (hereafter called \emph{temporal background knowledge}) is an issue that has not been systematically addressed in the literature. 

To tackle this open problem, we propose a novel optimisation based approach for generating \emph{temporally constrained counterfactual traces}, that is, counterfactual process traces that are guaranteed to comply with temporal background knowledge. To express such knowledge, we take the natural choice of adopting \LTLp, a temporal logic specifically designed for process traces~\cite{FioG18}, starting from the well established formalism of linear temporal logic on finite traces~\cite{DeGV13}. 
The backbone of our approach consists in infusing automata-theoretic techniques for \LTLp/\LTLf~\cite{DDMM22} within a Genetic Algorithm (GA) body. In particular, given an \LTLp formula $\varphi$ and a process trace $\trace$ satisfying $\varphi$, we introduce different strategies that inspect the automaton of $\varphi$ to suitably constrain the mutations introduced by the GA to alter $\trace$ towards counterfactual generation, so as to guarantee that the mutated version continues to satisfy $\varphi$.
%,  Specifically, we explore three strategies for incorporating LTLf-based temporal constraints into the GA: an \emph{a priori} strategy that extracts DFA information before running the GA, and two \emph{online} strategies that dynamically guide the GA during the counterfactual generation process.

An empirical evaluation on real-world and synthetic datasets demonstrates the effectiveness of our approach. Our results indicate that incorporating temporal background knowledge significantly improves the quality of counterfactual explanations, ensuring they satisfy \LTLp formulae of interest while satisfying general counterfactual desiderata.

\section{Running Example}
\label{sec:motexample}
Consider a scenario in which an estate agency is using a PPM system to forecast the renting of flats to its customers during the application and negotiation process. The system takes in input an ongoing trace and forecasts \textbf{Accept} in case of a successful application and negotiation and \textbf{Fail} otherwise. For example, consider the ongoing trace $\trace_1$, for which the predictive system forecasts a negotiation failure: \\[.4em]
\resizebox{\columnwidth}{!}{
$
\begin{aligned}
\trace_1 & = \submitapplication, \automaticcheck, \manualcheck, \informphone, \accept, \createoffer, \informphone, \askappointment 
\end{aligned}
$
}\\[.4em]
The trace starts with the customer submitting the application (\submitapplication) and the agency running an automatic check (\automaticcheck). Following a failure, a manual check (\manualcheck) is done, and the  customer is informed by phone (\informphone) of the application status, which is — in this case —  accepted as valid (\accept). An offer is created (\createoffer) and communicated to the customer by phone (\informphone). The customer then asks for an appointment (\askappointment) to discuss the offer.

To understand the reasons behind the prediction of \textbf{Fail} for $\trace_1$, the estate agency intends to obtain counterfactual explanations for $\trace_1$, that is, ongoing traces that suggest how to modify $\trace_1$ so that a successful negotiation is predicted. Examples are reported below, where \senddocument indicates that the customer sends the required documents to the estate agency, and \informsms (resp., \informemail) captures that the agency informs the customer via sms (resp., e-mail):\\[1em]
\resizebox{\columnwidth}{!}{
$\begin{aligned}
\trace_{c_1} &= \submitapplication, \wrongtask{\manualcheck}, \wrongtask{\automaticcheck}, \informphone, \accept, \createoffer, \informphone, \senddocument \\
\trace_{c_2} &= \submitapplication, \automaticcheck, \manualcheck, \informphone, \accept, \createoffer, \informphone, \senddocument  \\
\trace_{c_3} &= \submitapplication, \automaticcheck, \wrongtask{\informphone}, \accept, \createoffer, \wrongtask{\informsms}, \senddocument \\
\trace_{c_4} &= \submitapplication, \automaticcheck, \informphone, \informphone, \accept, \createoffer, \informphone, \senddocument, \askappointment \\
\trace_{c_5} &= \submitapplication, \automaticcheck, \manualcheck, \wrongtask{\informphone}, \accept, \createoffer, \wrongtask{\informemail}, \senddocument
\end{aligned}$
}\\[.5em]
We assume that counterfactual explanations $\trace_{c_1}, \ldots, \trace_{c_5}$ are properly built using available activities in the application and negotiation process. Nonetheless, some of them must be ruled out when considering how the agency operates. In our example, the agency has operational rules that state that: (i) within a negotiation, an automated check must eventually be conducted, and in case also a manual check is done, it can only be done after the automated one; and (ii) the agency always informs the applicant with the same communication mode (i.e., it is not possible to have, in the same negotiation, communications done with distinct modes). These two rules can be expressed using \LTLp (recalled next) as follows: \\[1em]
\resizebox{\columnwidth}{!}{
$\begin{aligned}
\formulacheck &= (\neg \manualcheck) \until \automaticcheck \\
\formulacomm &=  \neg((\eventually~\informphone) \land (\eventually~\informsms))
  \land \neg((\eventually~\informphone) \land (\eventually~\informemail)) \land \\
  & \qquad \neg((\eventually~\informsms) \land (\eventually~\informemail))
\end{aligned}
$
}\\[.5em]
It turns out that, among $\trace_{c_1}, \ldots, \trace_{c_5}$, $\trace_{c_1}$ violates $\formulacheck$, while $\trace_{c_3}$ and $\trace_{c_5}$ violate $\formulacomm$. The activities causing the violations are underlined in $\trace_{c_1}$, $\trace_{c_3}$ and $\trace_{c_5}$.

All in all, dealing with the presented scenario challenges state-of-the-art counterfactual generation approaches. In fact, when these approaches are applied in an out-of-the-box way, they cannot incorporate the background knowledge captured in $\formulacheck$ and $\formulacomm$. What we need is a technique that only considers, as counterfactual traces, those that ensure to respect the temporal background.

% of the temporal relations between the activities, and possible constraints on their sequentiality. This highlights the importance of considering the system's specification when using counterfactual generation techniques applied to temporal data.

%%%%%%%%%%%%%%%%%%%%%%%%%%%%%%%%%%
%!TEX root = ./main.tex

\section{Background}
\label{sec:background}
We overview here the essential background for the paper.

\subsection{Linear Temporal Logic on Process Traces}
\label{sec:LTLf}

Linear temporal logics, that is, temporal logics predicating on traces, are the most natural choice to express background knowledge in our setting. Traditionally, traces are assumed to have an infinite length, as witnessed by the main representative of this family of logics, namely \LTL~\cite{Pnue77}. 
In several application domains, such as those mentioned in \S\ref{sec:intro} and \S\ref{sec:motexample}, the dynamics of the system are more naturally captured using unbounded, but \emph{finite}, traces \cite{DeDM14}. This led to \emph{\LTL over finite traces} (\LTLf) \cite{DeGV13}, which adopts the syntax of \LTL but interprets formulae over finite traces. In our work, we are specifically interested in a variant of \LTLf where propositions denote atomic activities constituting the basic building blocks of a process, and where each state indicates which atomic activity has been executed therein. 
%As surveyed in \cite{DeDM14}, this variant has been extensively used in AI and BPM. 
This logic has been termed \emph{\LTL on process traces} (\LTLp) in \cite{FioG18}, which we follow next.

Fix a finite set \tasks of activities. A \emph{(process) trace} \trace over \tasks is a finite, non-empty sequence $\atask_1,\ldots,\atask_{n}$ over \tasks, indicating which activity from $\Sigma$ is executed in every instant $i \in \set{1,\ldots,n}$ of the trace. The length $n$ of $\trace$ is denoted $\length(\trace)$. For $i<\length(\trace)$, $\trace(i)$ denotes the activity $\atask_i$ executed at instant $i$ of $\trace$, and $\trace(\colon i)$ the prefix $a_1,\ldots,a_i$ of $\trace$.

An \LTLp formula $\varphi$ is defined according to the grammar
{
\begin{equation*}
\begin{aligned}
\varphi & ::=  \atask \mid \lnot \varphi \mid \varphi_1 \lor \varphi_2 \mid \tomorrow \varphi \mid \varphi_1 \until \varphi_2, \text{ where $\atask \in \Sigma$.}
\end{aligned}
\end{equation*}
}
Intuitively, $\tomorrow$ is the \emph{(strong) next} operator: $\tomorrow \varphi$ indicates that the next instant must be within the trace, and $\varphi$ is true therein. $\until$ is the \emph{until} operator: $\varphi_1 \until \varphi_2$ indicates that $\varphi_2$ is true now or in a later instant $j$ of the trace, and in every instant between the current one and $j$ excluded, $\varphi_1$ is true. Formally, given an \LTLp formula $\varphi$, a process trace $\trace$, and an instance $i \in \set{1,\ldots,\length(\trace)}$, we inductively define that $\varphi$ is true in instant $i$ of $\trace$, written $\trace,i\models \varphi$, as:

{\fontsize{9pt}{9pt}\selectfont
\begin{align*}
\trace,i & \models  \atask && \text{if } \trace(i) = \atask\\
 \trace,i & \models  \lnot \varphi && \text{if } \trace,i \not \models \varphi\\
\trace,i & \models \varphi_1 \lor \varphi_2 && \text{if } \trace,i \models \varphi_1 \text{ or }  \trace,i \models \varphi_2\\
\trace,i & \models \tomorrow \varphi &&
\text{if } i+1 \leq \length(\trace) \text{ and } \trace,i+1 \models \varphi \\
\trace,i & \models \varphi_1 \until \varphi_2 
&&
\text{if } \trace,j \models \varphi_2 \text{ for some } j \text{ s.t.~} i \leq j \leq \length(\trace) \\
&&&
\text{and } \trace,k \models \varphi_1 \text{ for every } k \text{ s.t.~} i \leq k < j
\end{align*}}
We say that $\trace$ satisfies $\varphi$, written $\trace \models \varphi$, if $\trace,1 \models \varphi$. 

The other boolean connectives $\true$, $\false$, $\land$, $\limp$ are derived as usual, noting $\true = \bigvee_{\atask_i \in \tasks} \atask_i$. Further key temporal operators are derived from $\tomorrow$ and $\until$. E.g.,
$\eventually \varphi = \true \until \varphi$ states that $\varphi$ is \emph{eventually} true in some future instant.
%, and $\always \varphi = \neg \eventually \neg \varphi$ (\emph{globally}) states that $\varphi$ is true in every future instant.

% The boolean connectives $\true$, $\false$, $\land$, $\limp$ are defined starting from $\lnot$ and $\lor$ as usual.  
% %\begin{inparaenum}[\itshape (i)]
% %\item $\true = \bigvee_{\atask_i \in \tasks} \atask_i$;
% %\item $\false = \neg true$;
% % \item $\varphi_1 \land \varphi_2 = \neg (\neg \varphi_1 \lor \neg \varphi_2)$;
% % \item $\varphi_1 \limp \varphi_2 = \neg \varphi_1 \lor \varphi_2$.
% % \end{inparaenum}
% In addition, other temporal operators are derived as follows:\todo{Marco possiamo eliminare la definizione estesa degli altri operatori citando?}
% \begin{inparaenum}[\itshape (i)]
% \item $\wtomorrow \varphi = \neg \tomorrow \neg \varphi$ --- \emph{weak next}, stating that if the next instant is within the trace, $\varphi$ is true therein;
% \item $\eventually \varphi = \true \until \varphi$ --- \emph{eventually}, stating that $\varphi$ is true in the future;
% \item $\always \varphi = \neg \eventually \neg \varphi$ --- \emph{globally}, stating that $\varphi$ is true in every future instant of the trace;
% \item $\varphi_1 \wuntil \varphi_2 = (\varphi_1 \until \varphi_2) \lor \always \varphi_1$ --- \emph{weak until}, weakening \emph{until} by allowing for the possibility that $\varphi_2$ does not become true in the future, and stating that
% in that case $\varphi_1$ will be true until the end of the trace.
% \end{inparaenum}
 
Since every \LTLp formula is an \LTLf formula, the automata-theoretic machinery defined for \LTLf \cite{DeGV13,DDMM22} applies to \LTLp as well. Specifically, we recall the definition of a deterministic finite-state automaton (\DFA) over process traces: a standard \DFA with the only difference that, due to how process traces are defined, it employs $\tasks$ instead of  $2^\tasks$ in labelling its transitions. a \DFA over process traces from $\tasks$ is a tuple $\aut = \tup{\tasks,\states,\istate,\trans,\fstates}$, where:
\begin{inparaenum}[\itshape (i)]
\item $\states$ is a finite set of states;
\item $\istate \in \states$ is the initial state;
\item $\trans: \states \times \tasks \rightarrow \states$ is the ($\tasks$-labelled) transition function;
\item $\fstates \subseteq \states$ is the set of final states.
\end{inparaenum}
A process trace $\trace = \atask_1,\ldots,\atask_n$ is accepted by $\aut$ if there is a sequence of $n+1$ states $\istate, \ldots, \state_{n}$ such that:
\begin{inparaenum}[\itshape (i)]
\item the sequence starts from the initial state $\istate$ of $\aut$;
\item the sequence culminates in a last state, that is, $\state_{n} \in \fstates$;
\item for every $i \in \set{1,\ldots,n}$, we have $\trans(q_{i-1},\atask_i) = q_{i}$. 
\end{inparaenum}
%The language $\lang(\aut)$ of $\aut$ is the set of process traces accepted by $\aut$. 
%
Notably, every \LTLp formula $\varphi$ can be encoded into a \DFA over process traces $\aut_\varphi$ that recognises all and only those traces that satisfy  $\varphi$: for every process trace $\trace$ over $\tasks$, we have that $\trace$ is accepted by $\aut_\varphi$ if and only if $\trace \models \varphi$.

Given an \LTLp formula $\varphi$, we denote by $\activetasks{\varphi}$ the set of activities mentioned in $\varphi$, and by $\othertasks{\varphi}$ the set of \emph{other} activities, that is, activities not mentioned in $\varphi$:  $\othertasks{\varphi} = \tasks \setminus \activetasks{\varphi}$.

%and use meta-symbol $\other$ to refer to any activity from $\tasks$ that is not mentioned in $\varphi$ (i.e., any activity in $\tasks \setminus \tasks_\varphi$). Intuitively, $\other$ is a placeholder for ``other'' activities. In the \DFA $\aut_\varphi$, a transition of the form $\delta(\state,\other) = \state'$ then denotes a set of transitions of the form $\delta(\state,\atask) = \state'$ for every activity $\atask$ not mentioned in $\varphi$. 
%
\btext{Figure~\ref{fig:dfa} shows the \DFA of \LTLp formula $\formulacheck$ from \S\ref{sec:motexample}, which states that during negotiations, an \automaticcheck must be eventually done, and \manualcheck, if done, can only follow the automated one.}

\begin{figure}[t]
\centering
\resizebox{.6\columnwidth}{!}{
\begin{tikzpicture} [->,>=latex,auto,node distance=1cm and 2cm, very thick]

\node[istate,accepting] (s0) {$\state_0$};
\node[right = of s0] (mid) {};
\node[state, accepting, right = of mid] (s1) {$\state_1$};
\node[state, below = of mid] (s2) {$\state_2$};

\path
(s0) edge [loop below] node [left,xshift=-2mm] {$\othertasks{\varphi_{\mathit{chk}}}$} (s0)
(s0) edge node {$\automaticcheck$} (s1)
(s1) edge [loop right] node [right, align=center] {$\othertasks{\varphi_{\mathit{chk}}}$} (s1)
(s1) edge [loop, out=-30, in=-60, looseness=8] node [right] {$\automaticcheck$} (s1)
(s1) edge [loop, in = -120, out = -150, looseness=8] node [below, align=center] {$\manualcheck$} (s1)
(s0) edge node[sloped] {$\manualcheck$} (s2)
(s2) edge [loop left] node [left, align=center] {$\automaticcheck$} (s2)
(s2) edge [loop above] node [right,xshift=1mm, align=center,font=\small] {$\othertasks{\varphi_{\mathit{chk}}}$} (s2)
(s2) edge [loop right] node [right, align=center] {$\manualcheck$} (s2)
;

%
%\path
%(s0) edge node [pos=0.2, below left, align=center] {$true$} (s1)
%
%(s1) edge node [pos=0.5] {$a$} (s2)
%      edge node [above] {$\neg a$} (s3)
%
%(s2) edge node [above] {$\neg b$} (s4)
%      edge node [pos=0.5] {$b$} (s3)
%
%(s3) edge [loop below] node {$true$} (s3)
%
\end{tikzpicture}
}
\caption{Graphical representation of the \DFA for the \LTLp formula $\formulacheck = (\neg \manualcheck) \until \automaticcheck$. Each transition labelled $\othertasks{\varphi_{\mathit{chk}}}$ is a placeholder for the set of transitions connecting the same pair of states, one per activity in the set $\othertasks{\varphi_{\mathit{chk}}}$ (i.e., different from $\manualcheck$ and $\automaticcheck$).}
\label{fig:dfa}
\end{figure}

\subsection{Counterfactual Explanations}
\label{sec:cfs}
 Differently from other XAI methods,  %belong to the family of XAI methods. However, compared to other types of XAI methods %, such as feature attribution methods
%counterfactuals 
counterfactual explanations do not attempt to explain the inner logic of a predictive model, but rather offer alternatives
to the user to obtain a desired prediction~\cite{wachter2017counterfactual}.
%When dealing with black-box models, indeed, the internal logic of a model $h_{\theta}$ mapping a sample $\trace$ to a label $y$ (also called class value) is unknown, or otherwise uninterpretable to humans.

Given a black-box classifier $h_{\theta}$, a counterfactual $\trace_c$ of $\trace$ is a sample for which the prediction of the model is different from the one of $\trace$ (i.e., $h_{\theta}(\trace_c) \neq h_{\theta}(\trace)$)\btext{, such that the difference between $\trace$ and $\trace_c$ is \emph{minimal}~\cite{guidotti2022counterfactual}}. 
A counterfactual explainer is a function $F_t$, where $t$ is the number of requested counterfactuals, such that, for a given sample $\trace$, a black box model $h_{\theta}$, and the set $\mathcal{T}$ of samples used to train the black-box model, returns a set $\mathcal{C} = \{\trace_{c_1}, \ldots, \trace_{c_h}\} $ (with $h \leq t$).
For instance, from the running example, with $t = 5$, $y = \textbf{Fail}$, $y' = \textbf{Accept}$, and $\trace = \trace_1$, running $F_t(\trace, h_{\theta}, \mathcal{T})$ yields the counterfactuals $\trace_{c_1}, \ldots, \trace_{c_5}$.

The XAI literature outlines several desiderata for counterfactual explanations~\cite{verma2020counterfactual}:
(i) \emph{Validity:} Counterfactuals should %change features to
flip the original prediction, aligning with the desired class.
(ii) \emph{Input Closeness:} Counterfactuals should minimize changes for clearer explanations.
(iii) \emph{Sparsity:} Counterfactuals should alter as few attributes as possible for conciseness.
(iv) \emph{Plausibility:} Counterfactuals must adhere to observed feature correlations, ensuring feasibility and realism.
%(v) \emph{Causality:} Counterfactuals should respect known causal relations to be realistic and actionable.
(v) \emph{Diversity:} A set of counterfactuals should provide diverse alternatives for the user.

The \emph{validity} of a counterfactual $\trace_c$ (desideratum (i)) is measured by the function \emph{val}, which evaluates the difference between the predicted value $h_\theta(\trace_c)$ and the desired class $y'$:
\btext{\begin{align}
\label{eq:validity}
\text{val}(h_{\theta}(\trace_c), y') = I_{h_{\theta}(\trace_c) \neq y'}
\end{align}}
\massi{where $I$ is the indicator function.}\footnote{\label{ft:indicator}The indicator function $I_{x\neq y}$ is 1 when $x\neq y$ and 0 otherwise.}

%\begin{align}
%\label{eq:validity}
%\text{val}(h_{\theta}(\trace_c), y') = |h_{\theta}(\trace_c) = y'|
%\end{align}
%This measures the closeness between the predicted value $h_\theta(\trace_c)$ and the counterfactual label $y'$.

\emph{Input closeness} of the $\trace_c$ to $\trace$ (desideratum (ii)), measured by the \emph{dist} function, assesses their dissimilarity:
\begin{align}
\label{eq:dist}
\text{dist}(\trace, \trace_c) = \frac{1}{\length(\trace)} \sum_{i=1}^{\length(\trace)} d(\trace(i), \trace_{c}(i)) %\\
%d(\trace(i), {\trace_{c}(i)}) 
%&= I_{\trace(i) \neq \trace_{c}(i)}}
\end{align}
where $d$ is a properly defined distance in the feature space. In this work, for the distance $d$ between trace elements, we use the indicator function $d(x,y)=I_{x\neq y}$.

The \emph{sparsity} of $\trace_c$ regarding $\trace$ (desideratum (iii)) is measured through the $\textit{spars}$ function, which counts the number of changes in the counterfactual:
\begin{equation}
\label{eq:spars}
{
\text{spars}(\trace, \trace_c) = ||\trace - \trace_c||_0 = \sum_{i=1}^{\length(\trace)} I_{\trace(i) \neq \trace_{c}(i)}}
\end{equation}
%where $I$ is the indicator function.

\emph{Implausibility} (desideratum (iv)) of $\trace_c$ is measured by the \emph{implaus} function, calculating the distance between $\trace_c$ and the closest sample $\trace_z$ in the reference population $\mathcal{T}$:
\begin{equation}
\label{eq:implaus}
{
\text{implaus}(\trace_c,\mathcal{T}) =  \min_{\trace_z \in \mathcal{T}} \frac{1}{\length(\trace)} \sum_{i=0}^{\length(\trace)} d(\trace_{z}(i), \trace_{c}(i))}
\end{equation}
where $d$ is the same distance used in~\eqref{eq:dist}.
%where $d$ is the distance defined above.
%Here $\trace_{z}(i)$ corresponds to $\atask_j$ in $\trace_z$ and $\trace_{c}(i)$ corresponds to $\atask_j$ in $\trace_c$.

%The \emph{causality} (desideratum (v)) deals with the interdependence of features in the data, as counterfactuals should respect known causal relations between features to be both realistic and actionable. %In the case of process traces, we can define this desideratum as whether the generated counterfactuals in $\mathcal{C}$ satisfy the given temporal background knowledge. 
%We define this desideratum in Sec.~\ref{sec:approach}.

The fifth desideratum (v), \emph{diversity}, measures the pairwise distances between the counterfactuals in $\mathcal{C}$, using the $\textit{dive}$: 
% \begin{equation}
% \text{dive}(\mathcal{C}) = \frac{1}{|\mathcal{C}| (|\mathcal{C}| - 1)} \sum_{1 \leq i < j \leq |\mathcal{C}|} \text{dist}(\trace_{c},\trace_{c'})
% \end{equation}
%where $\trace_{c}$ and $\trace_{c'}$ are two counterfactuals from the set $\mathcal{C}$.
\begin{equation}
\label{eq:dive}
{
\text{dive}(\mathcal{C}) = \frac{1}{|\mathcal{C}| (|\mathcal{C}| - 1)} \sum_{\{\trace_{c},\trace_{c'}\}} \text{dist}(\trace_{c},\trace_{c'})}
\end{equation}
where \emph{dist} is defined in \eqref{eq:dist} and the sum is over all possible unordered pairs $\{\trace_{c},\trace_{c'}\}$ of elements $\trace_{c},\trace_{c'}$ of $\mathcal{C}$.

%%%%%%%%%%%%%%%%%%%%%%%%%%%%%%%%%%%%%%%%%%%%%%%%%%%%%%
\subsection{Genetic Algorithms}
\label{sec:GA}
Genetic Algorithms (GAs) 
are powerful optimisation techniques, inspired from the natural processes of evolution, widely used for complex problems due to their effectiveness~\cite{geneticbook}.
GAs work with a population of solutions $\mathcal{P}$, evaluating their quality through a \textit{fitness function}. We will present the main components of GAs through the prism of traces and counterfactual explanations.
Each candidate solution in the search space, such as each possible trace, is described by a set of genes, forming its \textit{chromosome} or \textit{genotype}. Below, we outline the main components of a GA.

%\textbf{Genotype Representation:}
The \textit{chromosome} $\trace_c$ of a candidate solution is a sequence of genes, where $\trace_c(i)$ refers to the $i$-th gene,
$i \in \set{1,\ldots,n}$ where $n$ is the length of the chromosome. % fixed throughout the entire population $\mathcal{P}$.
In the case of traces $\trace_c(i)$ represents the $i$-th executed activity.

 %If a candidate solution is represented by the trace $\trace_{c_1}$, its genotype is expressed as $\trace_{c_1}  = (\trace_{c_1}(1), \ldots, \trace_{c_1}(f))$ where $f$ represents the length of the chromosome, fixed throughout the entire population $\mathcal{P}$. 
%where each $\trace_{c1}(i)$ corresponds to the activity at time step $i$ in the trace.
%\textbf{Fitness Function:}

The \textit{fitness function} $f(\trace_c)$ evaluates the quality of each candidate solution $\trace_{c} \in \mathcal{P}$, providing a measure %of how well the candidate meets the optimization criteria.
that is used as the objective for the GA optimization.
%For example, in counterfactual generation, the fitness function maximizes the likelihood that the trace changes the outcome from a failure to an acceptance: $f(\trace_c) = \text{Score}(\trace_c)$ where $\text{Score}(\trace_c)$ quantifies the different objectives instantiated, such as the extent to which the trace $\trace_c$ achieves the desired outcome.
For example, in counterfactual generation, the fitness function instantiates the different objectives introduced in Sect.~\ref{sec:cfs}.%, such as the extent to which the trace $\trace_c$ achieves the desired outcome.

%\textbf{Initial Population:} 

The \textit{initial population} $\mathcal{P}$ is a set of candidate traces generated at the beginning of the GA process. It may consist of randomly generated or predefined candidates:
$
\mathcal{P} = \{\trace_{c_1}, \trace_{c_2}, \ldots, \trace_{c_p}\}
$
where $\trace_{c_k}$ denotes an individual candidate.
%\textbf{Selection Operator:}
Afterwards, a subset of the population $\mathcal{P}$ is selected based on the fitness score. 
%The probability of selecting a chromosome $\trace_c$ is typically proportional to its fitness $f(\trace_c)$.
%\textbf{Crossover:}
Selected chromosomes (parents) undergo \emph{crossover} to produce offsprings. This operation combines parts of two parents' chromosomes to create new ones, promoting genetic diversity. If $\trace_{p_1}$ %$  = (\trace_{p_1}(1), \ldots, \trace_{p_1}(f))$
and $\trace_{p_2}$ %$ = (\trace_{p_2}(1), \ldots, \trace_{p_2}(f))$
are two parents, %then either a single point or multiple points are chosen, such as $k$ for the single point, and the offspring is: $\trace_{o} = (\trace_{p1}(0), \ldots, \trace_{p1}(k), \trace_{p2}(k+1), \ldots, \trace_{p1}(f))$.
then each component of the offspring  $\trace_{o}$ is selected from either of the two, i.e.~$\trace_{o}(i)=\trace_{p_1}(i)$ or $\trace_{p_2}(i)$.

%\textbf{Mutation:}

Offspring chromosomes are subject to \emph{mutation}, which involves randomly altering one or more genes. For example, a gene $\trace_{o}(i)$ might change with some small probability $p_{\text{mut}}$, introducing new genetic material and helping to prevent premature convergence.
% \[trace_{o}'(i)= 
% \begin{cases}
%     \text{random gene} & \text{with probability } p_{\text{mut}} \\
%     \trace_{o}(i) & \text{otherwise}
% \end{cases}
% \]
Specifically, $\trace_{o}'(i)$ is mutated to a random gene with probability $p_\text{mut}$ and remains equal to $\trace_{o}(i)$ otherwise.
Once this offspring undergoes mutation ($\trace_{o}'$), it becomes part of the new population as $\trace_{c_k}$.

In our example, $\mathcal{P}$ consists of counterfactual candidates $\mathcal{P} = \{\trace_{c_1}, \ldots, \trace_{c_5}\}$. Each candidate's genotype, e.g., $\trace_{c_1}$, is an activity sequence, with each activity as a gene (e.g., $\trace_{c_1}(1) = \submitapplication$). The fitness function evaluates candidates based on factors like the ability to flip the outcome, where the fittest are selected for crossover and mutation.
Ensuring compliance with constraints, like maintaining the correct sequence of activities (\automaticcheck before \manualcheck), is crucial to generating compliant counterfactuals, such as $c_2, c_4$. 

%%%%%%%%%%%%%%%%%%%%%%%%%%%%%%%%%%%%%%%%%%%%%%%%%%%%%%%%%%%%%%%%%%%%%%%%

%%%%%%%%%%%%%%%%%%%%%%%%%%%%%%%%%%%%%%%%%%%%%%%%%%%%%%%%%%%%%%%%%%%%%%%%
%!TEX root = ./main.tex

%%%%%%%%%%%%%%%%%%%%%%%%%
\section{Approach}
\label{sec:approach}
We are now ready to introduce our framework for generating counterfactual explanations that comply with background knowledge described through \LTLp formulae.
%We combine genetic algorithms (GAs) with deterministic finite-state automata (\DFA), converting \LTLf formulae into DFA. This ensures robust and faithful counterfactual explanations that satisfy the given temporal background knowledge.
%Both methods guarantee the \LTLp formulae are satisfied while maintaining the diversity of the counterfactual explanations.

The first step of is to define a new \emph{compliance} desideratum (desideratum vi). This is done by using a \(\text{compliance}\) function measuring whether the counterfactual $\trace_c$ satisfies a \LTLp formula $\varphi$ representing the temporal background knowledge. A counterfactual $\trace_c$ is deemed as a \emph{temporally constrained counterfactual} if it satisfies $\varphi$. Formally:
\begin{equation}
\label{eq:conf}
\text{compliance}(\trace_c, \varphi) =
\begin{cases}
1 &\text{if } \trace_c \models \varphi\\
0 &\text{otherwise}.
\end{cases}
\end{equation}

Recalling $\varphi = \formulacheck \land \formulacomm$ and counterfactuals $\trace_{c_1}, \ldots, \trace_{c_5}$ from \S\ref{sec:motexample}, we have $\text{compliance}(\tau_c,\varphi)=1$ for $c \in \{c_2,c_4\}$ and  $\text{compliance}(\tau_c, \varphi) = 0$ for $c \in \{c_1,c_3,c_5\}$.

Next we formulate the fitness function and its afferent objectives (\S\ref{sec:problemformulation}), and we introduce the modified crossover and mutation operators for the generation of counterfactuals that guarantee the compliance to \LTLp formulae (\S\ref{sec:crossover_mutation}).

%%%%%%%%%%%%%%%%%%%%%%%%%%%%%%%%
\subsection{Optimisation Problem Formulations}
\label{sec:problemformulation}

We follow GA-based methods like~\cite{geco,moc} and instantiate the first four desiderata from \S\ref{sec:cfs} into corresponding optimisation objectives for the fitness function, including the \emph{compliance} desideratum.\footnote{In GAs, diversity is managed through selection, crossover, and mutation operators, rather than the fitness function.}

Hence, the objectives to optimise are:
\emph{validity} of the counterfactual $\trace_c$ \eqref{eq:validity}; the \dist of $\trace_c$ to the original trace $\tau$ \eqref{eq:dist}; \spars, quantifying the number of changes in $\trace_c$ \eqref{eq:spars} from $\trace$;
\impl, that corresponds to the distance of $\trace_c$ from the reference population $\mathcal{T}$ \eqref{eq:spars}; and \emph{compliance}, measuring whether the $\trace_c$ is compliant to $\varphi$ or not \eqref{eq:conf}.
The resulting fitness function $\mathbf{f}$ is thus defined as:
\begin{equation}
\label{eq:sobjadapted}
\begin{split}
\mathbf{f} = &
\text{val}(h_{\theta}(\trace_c),y') + \alpha \ \text{dist}(\trace,\trace_c)
+ \beta \ \text{spars}(\trace,\trace_c) +
\\ & \gamma \  \text{implaus}(\trace_c,\mathcal{T})
+ \delta \ \text{compliance}(\trace_c, \varphi).
\end{split}
\end{equation}
where $\alpha, \beta, \gamma, \delta$ are weighting factors controlling the influence of each term on the overall fitness.

%%%%%%%%%%%%%%%%%%%%%%%%%%%%%%%%%%%
\subsection{Temporal Knowledge-Aware Operators}
\label{sec:crossover_mutation}
To guarantee the satisfaction of the background knowledge in the form of $\varphi$, we modify the GA, specifically the crossover and mutation operators introduced in \S\ref{sec:GA}. 

\paragraph{Temporal Knowledge-aware Crossover} %\todo{Ma le parent traces devono anceh soddisfare $\varphi$? O non serve? Se serve, va messo qui ma anche nel teorema 1!->Non serve}
Given an original query instance (process trace) $\trace$ satisfying \LTLp formula $\varphi$, and two parent traces $\trace_{p_1}$ and $\trace_{p_2}$ in the current population $\mathcal{P}$, the Temporal Knowledge-aware Crossover operator presented in Algorithm~\ref{alg:crossoveroperator} generates an offspring individual $\trace_o$ that satisfies $\varphi$. It takes as input also the crossover probability $p_\text{c}$ and the alphabet $\tasks_\varphi$ of the activities mentioned in $\varphi$.
%set of activities $\Sigma$ extracted from the LTL$_f$ formula, as well as the indices for the features regarding the temporal aspect of the input $I_{TF} = \{i \mid x_i \in \mathbf{x} \text{ is a temporal feature} \}$.
%
The Temporal Knowledge-aware Crossover operator initiates an offspring individual $\trace_o$ by retaining from 
%the original query instance 
$\trace$ the %temporal features 
phenotype that actively interacts with $\varphi$, guaranteeing its satisfaction. This is formed by those activities in %the language of $\varphi$ set 
$\tasks_\varphi$ (lines \ref{algo1:initialization_starts}--\ref{algo1:initialization_ends}).
%This ensures that $\varphi$ will be satisfied in $\trace_o$.
%that if $\varphi$ is satisfied in $\trace$, it is satisfied also in $\trace_o$.
%the LTL$_f$ specifications satisfied in $\mathbf{x}$ is also satisfied in $\mathbf{O}$.
%We define $I_\varphi =\{i |  i \in [1,|\mathbf{x}|], \mathbf{x}_i \in \Sigma\}.$ as the set of indexes of $\mathbf{x}$ corresponding to the retained features, i.e., the activities belonging to the set $\Sigma$.
%\massi{ We define as $I_\varphi \subseteq I_{TF}$ the set of indices of these retaining features: $I_\varphi = \{i \in I_{TF} |  x_i \in \mathbf{x}, x_i \in \Sigma\}.$}
%\andrei{Moreover, we save the indices where we impose on the offspring $\mathbf{O}$ the part of $\mathbf{Q}$ that satisfies $\varphi$ as we will need that for the mutation operator}.
% Next, for each feature, a random probability $p$ is sampled for the choice between the two parents $\mathbf{P_1}$ or $\mathbf{P_2}$ (line 8).
The empty genes 
in the offspring individual chromosome are then filled with one of the two parents' genetic material,
%but only if the parent's \andrei{temporal} features do not belong in the $\varphi$ activities, that is, they do not belong to the set $\Sigma$ (lines 8-18).
but only if the corresponding  parent's gene %is not a temporal feature belonging to the $\varphi$ activities, i.e. 
does not interact with $\varphi$, i.e., contains an activity in from $\othertasks{\varphi}$ (lines \ref{algo1:crossover_starts}--\ref{algo1:crossover_ends}).
In detail, a random probability $p$ is sampled (line \ref{algo1:sampling_p}) for every empty gene (line \ref{algo1:empty_features}). The genetic material is then chosen, as in classical crossover operators, from either parent $\trace_{p_1}$ or $\trace_{p_2}$ (lines \ref{algo1:chooseP1}--\ref{algo1:chooseP2}), according to the given crossover probability $p_\text{c}$, if the parent's activity does not belong to $\tasks_\varphi$.
%with the additional condition that the parent's feature is not a temporal feature belonging to $\Sigma$.
% If $p$ is smaller than the given crossover probability $p_c$, and the corresponding $\mathbf{P_1}$'s feature is not a temporal feature belonging to the $\varphi$ activities, then the genetic material is selected from parent $\mathbf{P_1}$ (line \ref{algo1:chooseP1}), otherwise, if the $\mathbf{P_2}$'s feature is not a temporal feature belonging to the $\varphi$ activities, $\mathbf{P_2}$ is choosen, (line \ref{algo1:chooseP2}).
% the crossover probability ($p_c$), together with the drawn probability $p$, guides the random selection of genetic material from either parent $\mathbf{P_1}$ or $\mathbf{P_2}$ as in classical crossover operators, if the selected item is not a $\Sigma$ activity.
Otherwise, if both parents' activities at that gene %are temporal features belonging 
belong to $\tasks_\varphi$, the crossover operator uses the gene from the original query instance $\trace$ (line \ref{algo1:keepQ}). As we prove later, this ensures that $\trace_o$ alters $\trace$ in a way that does not affect the satisfaction of $\varphi$.

\begin{algorithm}[t]
\caption{{\small Temporal Knowledge-aware Crossover operation}}
\label{alg:crossoveroperator}
%{\fontsize{9pt}{9pt}\selectfont
\begin{algorithmic}[1]
\State \textbf{Input:} parent individuals $\trace_{p_1}$ and $\trace_{p_2}$, crossover probability $p_\text{c}$, original query instance $\trace$, activities $\tasks_\varphi$
\State \textbf{Output:} offspring trace $\trace_{o}$

\For{$i$ from $1$ to $|\trace|$} \label{algo1:initialization_starts}
    \If{$\trace(i) \in \tasks_\varphi$}
        %\State 
        $\trace_{o}(i) \gets \trace(i)$
    \Else
        %\State 
         ~$\trace_{o}(i) \gets \textit{null}$
    \EndIf
\EndFor \label{algo1:initialization_ends}

\For{$i$ from $1$ to $|\trace_o|$} \label{algo1:crossover_starts}
    \State $p \sim U(0, 1)$ \label{algo1:sampling_p}
    \If{$\trace_{o}(i)$ is \textit{null}} \label{algo1:empty_features}
        \If{$p < p_\text{c} \wedge {\trace_{p_1}(i)} \notin \Sigma_\varphi$}
            %\State 
            $\trace_{o}(i) \gets \trace_{p_1}(i)$ \label{algo1:chooseP1}
        \ElsIf{$p \ge p_\text{c} \wedge {\trace_{p_2}(i)} \notin \Sigma_\varphi$}
            %\State 
            $\trace_{o}(i) \gets {\trace_{p_2}(i)}$ \label{algo1:chooseP2}
        \Else
            %\State 
            ~$\trace_{o}(i) \gets \trace(i)$ \label{algo1:keepQ}
        \EndIf
    \EndIf
\EndFor \label{algo1:crossover_ends}

\State \textbf{return} $\trace_{o}$
\end{algorithmic}%}
\end{algorithm}

%Figure~\ref{fig:adapted} showcases how the adapted crossover operator could work in the case of the motivating example of Bob introduced in Section~\ref{sec:motivating}. 
%Given Bob's case, two parents from the previous population and the sets $A$ and $T$, the adapted crossover operator starts by checking the part of Bob's case that is present in either of the two sets, and then imposes that onto the offspring. Specifically, it can be seen that the values of the control-flow features %attributes 
%\attr{Event 1}, \attr{Event 2} and \attr{Event 4} appear either in $T$ or in $A$,
%correspond to the activities in both $T$ and $A$, 
%thus the following activities \act{Create application}, \act{Submit documents}, and \act{Receive reminder} are directly imposed onto the offspring individual $\mathbf{O}$.

%For the other event features, the gene \attr{Event 3} = \act{Receive missing info email} is crossed over from $\mathbf{P_1}$, while \attr{Event 5} = \act{Update missing info} is crossed over from $\mathbf{P_2}$. For the non-control-flow features, since they are attributes that are not related to the control-flow of the trace, the genes are randomly chosen from either $\mathbf{P_1}$ or $\mathbf{P_2}$.

\paragraph{Temporal Knowledge-aware Mutation}
We constrain the mutation operator, designed to increase the  diversity of the population, with two strategies that maintain the diversity in the generated offsprings while ensuring that $\varphi$ is satisfied.

The first strategy, called \firstheur, computes all the possible mutations for each gene $\trace_{o}(i)$
%\footnote{Given a process trace $\trace_o$, $\trace_{o}(i)$ defines its $i$-th component.} 
at the beginning of the mutation phase. The second strategy, called \secondheur, exploits \DFA $\aut_\varphi$ to compute the possible mutations for the current gene $\trace_{o}(i)$ in the construction of $\trace_{o}$ taking into account the  already constructed partial trace $\trace_{o}(\colon i-1)$.

%While \secondheur is more flexible\footnote{Think, for instance, to a partial counterfactual $\trace_{o} = \automaticcheck$ and to a background knowledge of the form $\eventually \automaticcheck$. Since the background knowledge is already satisfied by $\trace_{o}$ no matter how the construction of the counterfactual continues, \secondheur is able to exploit this fact by providing more freedom to the mutation operator.}, our evaluation shows that \firstheur is often faster. 	 

We indicate with $\mathcal{D}_i$ the set of all possible activities that can occur at the $i$-th gene in any generated counterfactual, and define it as the set of all activities occurring in position $i$ in all historical traces $\mathcal{T}$. Formally, $D_i = \{\trace(i) \mid \trace \in \mathcal{T}\}$.
Given an offspring $\trace_{o}$, %
the \firstheur strategy produces an offspring $\trace'_{o}$ by mutating only genes that are not %present 
in $\Sigma_\varphi$ with values that are not %present 
in  $\Sigma_\varphi$, thus mapping \emph{other} activities into \emph{other} activities, which interact with $\varphi$ interchangeably. 
%\todo{Andrei: I shortned a touch here, please check!}
The \secondheur strategy instead produces an offspring $\trace'_{o}$ by computing, for every 
partial trace $\trace_o(\colon i)$ with $1 \leq i \leq |\trace_{o}|$, which activities could be used in place of the last activity $\trace_o(i)$ to alter such position $i$ without changing the satisfaction of $\varphi$. This is realized through the \textsc{Safeact} function defined in Algorithm~\ref{alg:safemoves}, exploiting the \DFA $\aut_\varphi$ as follows. First $\aut_\varphi$ is traversed using the sequence of activities in $\trace_o(\colon i-1)$, leading (deterministically) to a state $q$ of the \DFA. Then it is checked which next state $q'$ is obtained by applying transition $\trans(q,\trace_o(i))$. The safe activities that can be used in place of $\trace_o(i)$ in state $q$ are then those that lead to the same next state $q'$. In a sense, this generalises \firstheur, as in some states also activities from $\activetasks{\varphi}$ may be interchangeable.

\begin{algorithm}[t]
\caption{Compute Safe Activities}
\label{alg:safemoves}
%{\fontsize{9pt}{9pt}\selectfont
\begin{algorithmic}[1]
\Function{Safeact}{$\trace_o, i, A$}
    \State $\Sigma^{\text{safe}} \gets \{\}$
    
    % Initialize the state at the start of M
    \State $q \gets$ initial state of $A$
    
    % Compute the state q_{i-1} using trace_o(:i-1)
    \For{$j$ from $1$ to $i-1$}
        %\State 
        $q \gets \delta(q, \trace_o(j))$
   \EndFor

    % Compute the state q_i using trace_o(:i)
    %\State $q_i \gets q_{i-1}$
    \State $q' \gets \delta(q, \trace_o(i))$
    
    % Determine safe activities
    \For{$\atask \in \tasks$}
        \If{$\delta(q, \atask) = q'$}
         %\State 
         $\Sigma^{\text{safe}} \gets \Sigma^{\text{safe}} \cup \{\atask\}$
        \EndIf
    \EndFor
    
    \State \textbf{return} $\tasks^{\text{safe}}$
\EndFunction
\end{algorithmic}%}
\end{algorithm}

\begin{comment}
\begin{algorithm}[h!]
\caption{Compute Safe Activities}
\label{alg:safemoves}
\begin{algorithmic}[1]
\footnotesize
\Function{Safeact}{$\trace_o,i, M$}
    \State $A^{\text{safe}} = \{\}$
    \State $q_i = \textsc{State}(\trace_o(:i), M)$
    \State $q_{i-1} = \textsc{State}(\trace_o(:i-1), M)$
    \For{$\sigma \in \tasks$}
        \If{$\delta(q_{i-1}, \sigma) = q_i$}
            \State $A^{\text{safe}} = A^{\text{safe}} \cup \{\sigma\}$
        \EndIf
    \EndFor
    \State \textbf{return} $\tasks^{\text{safe}}$
\EndFunction
\end{algorithmic}
\end{algorithm}
%
\begin{algorithm}
\caption{State Function}
\label{alg:state}
\begin{algorithmic}[1]
\footnotesize
\Function{State}{$\trace, M$}
    \State $q_0 \gets$ initial state of $M$
    % \For{each symbol $a$ in $w$}
    \For{$i$ from $1$ to $|\trace|$}
        \State $q \gets \text{transition }\delta(q, \trace(i))$
    \EndFor
    \State \Return $q$
\EndFunction
\end{algorithmic}
\end{algorithm}
\end{comment}

%
\begin{algorithm}[tb]
\caption{{\small Temporal Knowledge-aware Mutation operator}}
\label{alg:mutation_operator}
%{\fontsize{9pt}{9pt}\selectfont
\begin{algorithmic}[1] 
\State \textbf{Input:} offspring $\trace_o$, mutation probability $p_\text{mut}$, \LTLp formula $\varphi$, domains of each gene $D$,
%Set of Temporal Features indices $I_{TF}$,
%Set of Imposed Indices $I_\varphi$,
 mutation strategy $S$
\State \textbf{Output:} mutated offspring $\trace_o'$
%    \State $p \sim U(0, 1)$
\For{$i$ from $1$ to $|\trace_o|$} \label{algo3:For_starts}
    \State $p \sim U(0, 1)$ \label{algo3:sample_p}
    \If{$p < p_\text{mut}$} \label{algo3:pm_threshold}
        %\If{$i \in I_{TF}$ } \label{algo3:If_TF}
        \If{$S$ is \firstheur  \textbf{ and } $\trace_{o}(i) \notin \tasks_\varphi$} \label{algo3:strategy_starts}
            \State $\trace_{o}(i) \sim U(
            \mathcal{D}_{i} 
            \setminus \Sigma_\varphi)$
        % \ElsIf{$S$ is \secondheur  \textbf{ and } $i \notin I_\varphi$}
        %     \State $A^{\text{safe}} = SafeActions(Q,\Sigma,\delta,F) $
        %     \State $\mathbf{O}_i \sim U((\mathcal{D}_{i} \setminus \Sigma) \cup A^{\text{safe}})$
        \ElsIf{$S$ is \secondheur}
            \State $\Sigma^{\text{safe}} \gets \textsc{SafeAct}(\trace_o, i, \aut_\varphi) $
            \State $\trace_{o}(i) \sim U(\D_i \cap \Sigma^{\text{safe}})$
        %\ElsIf{$S$ is \thirdheur}
        
        %    \State $\epsilon = \text{False}$
        %    \While{not $\epsilon$}
        %        \State $\trace_{o}(i) \sim U(\mathcal{D}_{i})$
        %        \State $\epsilon = \trace_o \models \varphi$ %\Comment Check satisfaction of $\varphi$
        %    \EndWhile
        \EndIf \label{algo3:strategy_ends}

    \Else \label{algo3:NoMut}
        %\State 
        ~$\trace_{o}(i) \gets \trace_{o}(i)$
            %\sim U(\mathcal{D}_{i})$
        %\EndIf \label{algo3:strategy_ends}
    \EndIf
\EndFor \label{algo3:For_ends}
\State $\trace'_o \gets \trace_o $
\State \textbf{return} $\trace'_o$
\end{algorithmic}%}
\end{algorithm}

The Temporal Knowledge-aware Mutation operator, presented in Algorithm~\ref{alg:mutation_operator}, focuses on mutating an offspring individual $\trace_o$ while preserving the satisfaction of the formula $\varphi$. 
The operator takes as input the offspring $\trace_o$, the mutation probability $p_\text{mut}$, the set of the domains of the genes $D=\{ \mathcal{D}_i \mid i \in \{1, \dots, |\trace_o| \} \}$, 
%the set of the activities to include $I$ and to exclude $E$, and 
%\massi{the temporal feature indices $I_{TF}$, the set of imposed indices $I_{\varphi}$ defined during the crossover operation,}
%and of the two mutation strategies $S$, 
and the chosen mutation strategy $S$,
%that is one of the two strategies defined above, 
returning the mutated offspring $\trace'_o$ as output. It also takes as input the \LTLp formula $\varphi$ capturing background knowledge, for which we assume that the \DFA $\aut_\varphi$ has been pre-computed (and hence is passed as implicit parameter to the algorithm).
%By domain of the features, we refer to the possible values that a certain attribute can take.
%
The algorithm starts by sampling a random mutation probability $p_\text{mut}$  (line \ref{algo3:sample_p}), and then iterates through each gene from 1 to $|\trace_o|$ (lines \ref{algo3:For_starts}--\ref{algo3:For_ends}).
%if they were not imposed from the original query instance $Q$ in Algorithm~\ref{alg:crossoveroperator},
The mutation is carried on if the sampled probability $p_\text{mut}$ is under the set threshold probability for mutation $p_\text{mut}$ (line \ref{algo3:pm_threshold}), otherwise, we return the value of $\trace_{o}(i)$ (line~\ref{algo3:NoMut}). 
%For temporal features (line \ref{algo3:If_TF}), 
In the case of mutation, the value of $\trace_o(i)$ is then uniformly sampled according to the selected mutation strategy $S$ %, from the set $\mathcal{D}_{i}$,
%together with the exclusion and inclusion criteria based on the chosen strategy,
%to avoid introducing new constraint activations and hence potential violations
(lines \ref{algo3:strategy_starts}--\ref{algo3:strategy_ends})%We avoid mutating imposed features as this may inevitably lead to a violation of the LTL$_f$ formula.
%\setminus \Sigma$, i.e., excluding activities in the set $A$ from the domain of the i-th feature to avoid introducing new constraint activations and hence potential violations (lines 4-5). 
%For non temporal features %that do not regard the elements of the LTL$_f$ formula ($\Sigma$ elements), if the mutation probability $p$ is lower than the predefined threshold $p_{mut}$,
%the value of $\mathbf{O}_i$ is sampled uniformly from the set $\mathcal{D}_{i}$ (lines \ref{algo3:If_NonTF}--\ref{algo3:If_NonTF_end}).
, mutating each gene accordingly, and returning the mutated offspring $\trace'_o$. % as the result.
%\todo{Andrei: Shortened a bit here, check}

\paragraph{Correctness of the approach.}
All in all, the application of the crossover operator from Algorithm~\ref{alg:crossoveroperator}, as well as that of the mutation operator with the two illustrated strategies from Algorithm~\ref{alg:mutation_operator}, are correct in terms of how they interact with temporal background knowledge, in the following sense.

%\footnote{\rtext{A rigorous proof is provided in the supplementary material.}}
%\todo{Il teorema va sistemato rispetto a come vanno definite le parent traces...}
\begin{restatable}{thm}{thmalgoone}
\label{thm:crossover}
Let $\varphi$ be a \LTLp formula, and $\trace_{p_1}$, $\trace_{p_2}$, and $\trace$ be process traces over $\tasks$, with $\trace \models \varphi$. Let $p_\text{c}\in\mathbb{R}_{[0,1]}$. Assume that Algorithm~\ref{alg:crossoveroperator} is invoked by passing $\trace_{p_1}$, $\trace_{p_2}$, $p_\text{c}$, $\trace$, and $\tasks_\varphi$ as input, and that it returns $\trace_o$. Then $\trace_o \models \varphi$.
\end{restatable}

\begin{proof}
\btext{Let $n = \length(\trace)$.
Upon inspection of Algorithm~\ref{alg:crossoveroperator}, one can see that every output $\trace_o$ produced by the algorithm  relates to the input trace $\trace$ as follows:
\begin{compactenum}
\item $\length(\trace_o) = \length(\trace) = n$;
\item for every $i \in \set{1,\ldots,n}$:
\begin{compactenum}
\item if $\trace(i) \in \activetasks{\varphi}$ then $\trace_o(i) = \trace(i)$;
\item if instead $\trace \not \in \activetasks{\varphi}$, that is, $\trace  \in \othertasks{\varphi}$, then $\trace_o(i) \in \othertasks{\varphi}$ as well -- equivalently, $\trace_o(i) \in \othertasks{\varphi}$ if and only if $\trace(i) \in \othertasks{\varphi}$.
\end{compactenum}
\end{compactenum}}

\btext{By absurdum, imagine that $\trace_o \not \models \varphi$. Since, by property (1) above, $\length(\trace_o) = \length(\trace)$, this means that the violation must occur due to a mismatch in the evaluation of an atomic formula in some instant. Technically, there must exist $i \in \set{1,\ldots,n}$ and an atomic sub-formula $a \in \Sigma$ of $\varphi$ (which, by definition, requires $a \in \activetasks{\varphi}$) such that 
either:
\begin{compactenum}[(A)]
\item $\trace,i \models a$ and $\trace_o,i \not \models a$, or 
\item $\trace,i \not\models a$ and $\trace_o,i \models a$. 
\end{compactenum}}

\btext{\noindent \emph{Case (A).} By the \LTLp semantics, we have $\trace(i) = a$ and $\trace_o(i) \neq a$. However, this is impossible: since $a$ belongs to $\activetasks{\varphi}$, then by property (2a) above, we have $\trace(i) = \trace_o(i)$.}

\btext{\noindent \emph{Case (B).} By the \LTLp semantics, $\trace,i \not \models a$ if and only if $\trace(i) = b$ for some $b \in \Sigma \setminus \set{a}$. There are two sub-cases: either $b \in \activetasks{\varphi}\setminus \set{a}$, or $b \in \othertasks{\varphi}$. In the first sub-case, impossibility follows again from the fact that, by property (2a) above, since $b \in \activetasks{\varphi}$, then $\trace_o(i) = \trace(i) =b$, which implies $\trace_o,i \not\models a$. In the second sub-case, impossibility follows from the fact that $\trace_o(i)$ cannot be $a$, since by property (2b), the fact that $\trace(i) \in \othertasks{b}$ implies that also $\trace_o(i) \in \othertasks{b}$.}
\end{proof}

\begin{restatable}{thm}{thmalgothree}
\label{thm:mutation}
Let $\varphi$ be a \LTLp formula, $\trace_o$ a process trace over $\tasks$ s.t.~$\trace_o \models \varphi$, $D=\{\mathcal{D}_i\}\subset \tasks^{|\trace_o|}$ the domains of each gene, and $p_\text{mut}\in\mathbb{R}_{[0,1]}$.
Assume that Algorithm~\ref{alg:mutation_operator} is invoked by passing 
$\trace_o,\varphi,p_\text{mut},D,S$ as input (with $S \in \set{\text{\firstheur,\secondheur}}$), and that it returns $\trace'_o$. Then $\trace'_o \models \varphi$.
\end{restatable}

\begin{proof}
\btext{Correctness of \firstheur is proven analogously of Theorem~\ref{thm:crossover}. Correctness of \secondheur derives directly from the correspondence between the traces that satisfy $\varphi$, and the traces accepted by the \DFA of $\varphi$, $\aut_\varphi$. From the definition of \DFA acceptance, we have that since trace $\trace_o$ satisfies $\varphi$, there is a sequence $\istate, \ldots, \state_{n}$ of states of $\aut_\varphi$, such that:
\begin{inparaenum}[\itshape (i)]
\item the sequence starts from the initial state $\istate$ of $\aut$;
\item the sequence culminates in a last state, that is, $\state_{n} \in \fstates$;
\item for every $i \in \set{1,\ldots,n}$, we have $\trans(q_{i-1},\trace_o(i)) = q_{i}$. 
\end{inparaenum}}

\btext{From Algorithm~\ref{alg:mutation_operator}, we have that trace $\trace_o$ is mutated through \secondheur by selecting an instant $i$, and allowing for replacing $\trace_o(i)$ with an activity $a$ returned from Algorithm~\ref{alg:safemoves}. From Algorithm~\ref{alg:safemoves}, we know that activity $a$ is returned if the following property holds: $\trans(\state_{i-1},a) = q_{i}$. This means that the mutated trace $\trace'_o$ that replaces $\trace_o(i)$ with $a$, and maintains the rest identical, is accepted by $\aut_\varphi$, with the same witnessing sequence of states $\istate, \ldots, \state_{n}$. From the correspondence between $\aut_\varphi$ and the \LTLp semantics of $\varphi$, we thus have $\trace'_o \models \varphi$.}
\end{proof}

%Proofs are reported in the supplementary material. They follow from language-theoretic properties of \LTLp and of corresponding structural properties of their \DFA{s}.

%The proof of the corollary follows from the theorem since Algorithm~\ref{alg:crossoveroperator} can be seen as a special case of Algorithm~\ref{alg:mutation_operator} with $S=\firstheur$, $\trace_o=\trace$, and $\mathcal{D}_i=\{\trace_{p_1}(i),\trace_{p_2}(i)\}$.

%In this way, we compare different strategies that make 

We illustrate the two proposed strategies with formula $\formulacheck$ of the running example of \S\ref{sec:motexample}. 
Based on the \DFA of the formula (Fig.~\ref{fig:dfa}), we know that, from the initial state $\state_0$, if \automaticcheck happens, %fires,
we reach the state $\state_1$, which is a final state, where we can repeat any activity in $\tasks$. In $\state_0$ we can %also 
perform any other activity in $\othertasks{\varphi_{chk}}$, i.e., different from \automaticcheck and \manualcheck, remaining in $\state_0$. However, if from $\state_0$ we perform \manualcheck, we reach $\state_2$ which is a dead-end state and thus $\trace_1$  violates 
$\varphi$.
% than $b$ to remain in the same state, and we can perform $b$ to reach a final state, i.e.~state $q_1$.
Given the trace $\trace_1$ from our running example,
consider the mutation of its fourth component $\trace_1(4)=\informphone$.
%Based on the \DFA, 
The two strategies for the mutation operator give the following mutation possibilities: in the \firstheur strategy we exclude activities in $\tasks_\varphi$ from the mutation, so we can mutate $\trace_1(4)$ with activities in $\mathcal{D}_4\setminus\Sigma_\varphi$; in the \secondheur strategy the transition in the \DFA associated with $\trace_1(4)$ is $\delta(\state_1, \trace_1(4)) = \state_1$. Since from $\state_1$ all activities give the same transition, that is,  $\delta(\state_1, \atask)=\state_1$ for every $\atask\in\Sigma$, we can mutate $\trace_1(4)$ with the entire $\mathcal{D}_4$. 

%%%%%%%%%%%%%%%%%%%%%%%%%%%%%%%%%%%%%%%%%%%%%%%%%%%%%%%%%%%%%%%%%%%%%%%%
\section{Evaluation}
\label{sec:eval}

To evaluate the approaches for generating counterfactuals explanations incorporating temporal background knowledge, we focus on answering the following questions:

%How do the proposed mutation strategies compare to the \mar method in terms of \run, \sat, and counterfactual quality?
%How do the proposed mutation strategies compare to \ltlfbase in terms of improving both counterfactual quality and \sat performance?
\begin{enumerate}
    \item[\textbf{RQ1}] How do the proposed methods compare with a standard genetic algorithm?
    \item[\textbf{RQ2}] How do the proposed methods compare with a baseline strategy of ``generation and check'', which ensures the satisfaction of the temporal background knowledge?
\end{enumerate}
Both questions are assessed in terms of generation time and quality of the generated counterfactuals (see below). 

The goal of \textbf{RQ1} is to assess whether the enforcement of temporal background knowledge has an impact on the quality of the generated counterfactuals when compared to traditional GA approaches. The goal of \textbf{RQ2} is to evaluate the proposed strategy, which integrates GAs and \DFA against an iterative combination of GAs and a checker function that generates the counterfactual and then checks compliance in a trial-and-error fashion. While a comparison with a standard GA in terms of compliance is out of scope of this paper, as these algorithms are not built to ensure this property, \textbf{RQ1} also enables us to discuss also this aspect.

\paragraph{Baselines, Datasets and Evaluation metrics}
\label{sec:metrics}
We introduce two baseline methods.
To answer the first research question, we use a standard GA, \ltlfbase (\ltlfbaseshort), which uses the fitness function Eq.~\eqref{eq:sobjadapted}, and the
standard crossover and mutation operators from~\cite{geco}~\footnote{\btext{For compatibility reasons, the same fitness function is also used in the proposed methods, \firstheur and \secondheur, even though for them the value of compliance is always one since the generated offsprings $\tau_o$ always satisfy $\varphi$.}}. To answer to the second research question we employ \mar (\marshort) which pairs the GA method with \sat through a trial-and-error mechanism: the generated trace is checked against $\varphi$, and the mutation is repeated until \sat.
%\massi{In summary, \ltlfbaseshort considers \sat solely as a fitness function objective, while \marshort employs the temporal knowledge-aware crossover and standard mutation operators but ensures \sat through a trial-and-error mechanism. We thus refer to \marshort as a \emph{baseline strategy of generation and check}, as it does not guarantee the satisfaction of $\varphi$.}

Experiments are conducted using three datasets commonly used in Process Mining, with details reported in Table~\ref{tab:datasets}: 
\emph{Claim Management}~\cite{rizzi2020lime} is a synthetic dataset pertaining to a claim management process, where accepted claims are labelled as \textit{true} and rejected claims as \textit{false};
\emph{BPIC2012}~\cite{bpic2012} and \emph{BPIC2017}~\cite{bpic2012}
two real-life datasets about a
loan application process, where traces with accepted loan offers are labelled as \textit{true}, and declined offers as \textit{false}.
%Details of the datasets are reported in Table~\ref{tab:datasets}.

\begin{table}[tbp]
\centering
\scalebox{0.935}{
{\fontsize{9pt}{9pt}\selectfont
\begin{tabular}{l@{\hskip 0.1in}S[table-format=5]@{\hskip 0.1in}ccl}
\toprule
\textbf{Dataset}  & {\textbf{Traces}} & \textbf{Avg. Len.}  & $|\tasks|$    & \textbf{Used Prefixes} \\ \midrule
Claim Management       & 4800             & 11                       & 16 & 7, 10, 13, 16 \\ %\hline
BPIC2012             & 4685             & 35                       & 36   & 20, 25, 30, 35                                                                                                      \\ %\hline
BPIC2017                        & 31413            & 35                       & 26   & 20, 25, 30, 35                                                                                            \\ \bottomrule
\end{tabular}}}
\caption{Summary of the dataset characteristics.}
\label{tab:datasets}
\end{table}

We conduct experiments with different datasets, trace lengths, 
as well as with \LTLp formulae with sizes of $\tasks_\varphi$. 
The experiments are performed over traces with variable prefix lengths, reported in Table~\ref{tab:datasets}, testing how techniques perform on average with varying amounts of information~\cite{irene}. %determine how the technique performs based on the amount of information available~\cite{irene,caisepaper}. 
To assess the impact of the number of generated counterfactuals, we test different settings by generating 5, 10, 15, and 20 counterfactuals~\cite{guidotti2022counterfactual}.  

To evaluate the counterfactual approaches, we use five metrics from~\cite{caisepaper}: \emph{Distance} \eqref{eq:dist}; lower is better, \emph{Sparsity} \eqref{eq:spars}; lower is better, \emph{Implausibility} \eqref{eq:implaus}; lower is better, \emph{Diversity} \eqref{eq:dive}; higher is better, and \emph{Runtime} in seconds; lower is better. These metrics refer to a single $\trace$ and are averaged across the set of generated counterfactuals.~\btext{In our experiments we omit the hit rate of the counterfactual set (i.e., whether $|\mathcal{C}| = t$), as in all our experiments the hit rate was always 100\%.} % for a single sample $\trace$.

\paragraph{Experimental procedure} 
For each dataset, \LTLp formula $\varphi$, and 
prefix length, we split the data into $ 70 \% - 10 \% - 20 \%$ %partitioning it 
into training, validation, and testing, using a chronological  order split. A %Extreme Gradient Boost (XGBoost) 
XGBoost model is trained and optimised using hyperparameter optimisation to identify the best model configuration for each dataset, prefix length, and encoding combination. 
The %training 
set $\mathcal{T}$ used to train the XGBoost model is used as input for the counterfactual generation methods.
We tested different \LTLp formulae $\varphi$, with different coverage percentages of the possible different activities $|\Sigma_\varphi|/|\Sigma|$~\footnote{\btext{Concerning the impact of the complexity of a formula on our work there are two distinct aspects to be considered. First, formula complexity impacts on the construction of the DFA, because the size of the DFA is, in the worst-case, doubly exponential in the length of the formula~\cite{DeGV13}. However, this is orthogonal to our approach, as we consider that the DFA has been already constructed~\cite{DeDM14,fuggitti-ltlf2dfa}. In addition, it is well-known that, despite this worst-case complexity, the size of the DFA is often polynomial in the length of the formula, and that such off-the-shelf techniques incorporate several optimizations~\cite{DGF21}. Second, it is less meaningful to relate metrics on the syntactic complexities of a formula with performance, as their interaction occurs at the semantic level. In this respect, focusing on coverage reflects the intuition that when a modeller explicitly mentions an activity in a formula, they do so to express constraints on when such an activity must/can be executed (i.e., a modeller would not express formula “true”, which has empty formula signature, with the equivalent formula $\bigvee_{a_i \in \Sigma} a_i$, which has the whole alphabet as formula signature). In this sense, using more activities (i.e., having a larger formula signature) correlates with “constraining more”, which in turn impacts on performance.
This motivates the choice of coverage, and in turn why in our experiments coverage influences performance.
}}.~\btext{The specific $\varphi$ formulas for each dataset are presented in the code repository linked in the beginning of the paper.}
\btext{Regarding the coefficients in Eq.~\eqref{eq:sobjadapted}, after testing multiple configurations, the final configuration was set to $\alpha= 0.5,\beta = 0.5,\gamma = 0.5,\delta = 0.5$ to give all objectives the same weight.}
%Traces that do not satisfy $\varphi$ are removed, 
%ensuring all traces are compliant.\footnote{We used~\citet{declare4py} for checking the compliance to the \LTLf formulas.}

Next, 15 instances are sampled from the test set and used for the counterfactual generation, one trace $\trace$ at a time, while the counterfactuals are evaluated using the evaluation framework.
Experiments were run on a M1 with 16GB RAM.
For the GA setting, we initialise the population through a hybrid approach: selecting close points from the reference population or, if unavailable, by randomly generating traces.
%As parameters, 
We set the number of generations to $100$, $p_\text{c} = 0.5$, $p_\text{mut} = 0.2$.
In \emph{population selection}, the top 50\% of the population, w.r.t.~the fitness function, moves to the next generation. Termination occurs at the max generation number or if no significant performance improvement occurs.

We assess differences using statistical tests: we perform Wilcoxon signed-rank tests~\cite{wilcoxon} for pairwise comparisons, with p-values adjusted by the Bonferroni correction~\cite{bonferri}. Methods are then ranked by performance on each metric, allowing for clear comparison.

%%%%%%%%%%%%%%%%%%%%%%%%%%%%%%%%%%%%%%%%%%%%%%%%%%%%%%%%%%%%%%%%%%%%%%%%
%!TEX root = ./main.tex

\section{Results}
\label{sec:results}

\begin{table*}[t]
\centering
\scalebox{0.7}{
\begin{tabular}{llcccccccccccc}
\toprule
\multirow{2}{*}{Cover.}  &  \multirow{2}{*}{Metric} & \multicolumn{4}{c}{Claim Management} & \multicolumn{4}{c}{BPIC2012} &  \multicolumn{4}{c}{BPIC2017}  \\ 
\cmidrule(lr){3-6} \cmidrule(lr){7-10} \cmidrule(lr){11-14}
 & & \ltlfbaseshort & \marshort & \firstheur & \secondheur  & \ltlfbaseshort  & \marshort & \firstheur & \secondheur  & \ltlfbaseshort & \marshort & \firstheur & \secondheur \\
\midrule
 \multirow{5}{*}{10\%} 
 %& Conf &
   %$\winner{1.0}\rank{1}$ & $\winner{1.0}\rank{1}$ & $\winner{1.0}\rank{1}$ & $\winner{1.0}\rank{1}$ &%SYNTH
   %$\winner{1.0}\rank{1}$ & $\winner{1.0}\rank{1}$ & $\winner{1.0}\rank{1}$ & $\winner{1.0}\rank{1}$ &%BPIC2012
   %$\winner{0.99}\rank{1}$ & $\winner{1.0}\rank{1}$ & $\winner{1.0}\rank{1}$ & $\winner{1.0}\rank{1}$ \\%BPIC17
 & Dist. &
  $\winner{0.48}\rank{1}$ & $\winner{0.50}\rank{1}$ & $\winner{0.50}\rank{1}$ & $\winner{0.50}\rank{1}$ &%SYNTH
  $0.54\rank{4}$ & $0.50\rank{2}$ & $\winner{0.47}\rank{1}$ & $0.50\rank{2}$ &%BPIC2012
  $0.59\rank{4}$ & $0.50\rank{2}$ & $\winner{0.44}\rank{1}$ & $0.50\rank{2}$ \\ %BPIC17
& Spars. &
   $\winner{2.50}\rank{1}$ & $\winner{2.58}\rank{1}$ & $\winner{2.63}\rank{1}$ & $\winner{2.58}\rank{1}$ &%SYNTH
   $7.69\rank{4}$ & $\winner{6.95}\rank{1}$ & $\winner{6.79}\rank{1}$ & $\winner{6.97}\rank{1}$ &%BPIC2012
   $6.95\rank{4}$ & $\winner{5.60}\rank{1}$ & $\winner{5.20}\rank{1}$ & $5.85\rank{3}$ \\%BPIC17
& Impl. &
   $\winner{8.49}\rank{1}$ & $8.88\rank{2}$ & $8.84\rank{2}$ & $8.92\rank{2}$ &%SYNTH
  $\winner{7.49}\rank{1}$ & $\winner{7.37}\rank{1}$ & $\winner{7.26}\rank{1}$ & $\winner{7.38}\rank{1}$ &%BPIC2012
  $\winner{7.51}\rank{1}$ & $\winner{6.81}\rank{1}$ & $\winner{6.59}\rank{1}$ & $\winner{6.86}\rank{1}$ \\%BPIC17
& Dive &
   $0.48\rank{4}$ & $\winner{0.54}\rank{1}$ & $\winner{0.54}\rank{1}$ & $\winner{0.54}\rank{1}$ &%SYNTH
   $0.35\rank{4}$ & $\winner{0.40}\rank{1}$ & $\winner{0.41}\rank{1}$ & $\winner{0.40}\rank{1}$ &%BPIC2012
   $\winner{0.58}\rank{1}$ & $0.49\rank{2}$ & $0.44\rank{4}$ & $0.49\rank{2}$ \\ %BPIC17
 & Runtime &
  $\winner{2.53}\rank{1}$ & $4.65\rank{4}$ & $4.38\rank{2}$ & $4.09\rank{2}$ &%SYNTH
   $25\rank{3}$ & $90\rank{4}$ & $\winner{12}\rank{1}$ & $17\rank{2}$ &%BPIC2012
   $\winner{29}\rank{1}$ & $78\rank{2}$ & $57\rank{2}$ & $65\rank{2}$ \\%BPIC17
\midrule
\multirow{5}{*}{25\%} 
 %& Conf &
   %$\winner{0.99}\rank{1}$ & $\winner{1.0}\rank{1}$ & $\winner{1.0}\rank{1}$ & $\winner{1.0}\rank{1}$ &%SYNTH
   %$\winner{0.99}\rank{1}$ & $\winner{1.0}\rank{1}$ & $\winner{1.0}\rank{1}$ & $%\winner{1.0}\rank{1}$ &%BPIC2012
   %$\winner{0.98}\rank{1}$ & $\winner{1.0}\rank{1}$ & $\winner{1.0}\rank{1}$ & $\winner{1.0}\rank{1}$ \\%BPIC17
& Dist. &
   $\winner{0.48}\rank{1}$ & $\winner{0.47}\rank{1}$ & $\winner{0.50}\rank{1}$ & $\winner{0.47}\rank{1}$ &%SYNTH
   $0.47\rank{2}$ & $0.38\rank{4}$ & $\winner{0.30}\rank{1}$ & $0.36\rank{2}$ &%BPIC2012
  $0.58\rank{4}$ & $0.33\rank{2}$ & $\winner{0.22}\rank{1}$ & $0.31\rank{2}$ \\%BPIC17
& Spars. &
   $2.50\rank{1}$ & $2.45\rank{1}$ & $2.58\rank{4}$ & $2.42\rank{1}$ &%SYNTH
   $7.01\rank{3}$ & $5.24\rank{3}$ & $\winner{4.38}\rank{1}$ & $4.92\rank{2}$ &%BPIC2012
   $6.77\rank{4}$ & $4.02\rank{2}$ & $\winner{2.72}\rank{1}$ & $3.82\rank{2}$ \\%BPIC17
& Impl. &
   $\winner{8.46}\rank{1}$ & $8.80\rank{4}$ & $\winner{8.83}\rank{1}$ & $\winner{8.70}\rank{1}$ &%SYNTH
   $\winner{6.38}\rank{1}$ & $9.73\rank{2}$ & $9.67\rank{2}$ & $9.70\rank{2}$ &%BPIC2012
   $7.39\rank{4}$ & $\winner{5.52}\rank{1}$ & $\winner{4.94}\rank{1}$ & $\winner{5.40}\rank{1}$ \\%BPIC17
& Dive &
   $0.47\rank{3}$ & $0.50\rank{2}$ & $\winner{0.52}\rank{1}$ & $0.49\rank{3}$ &%SYNTH
   $0.36\rank{2}$ & $\winner{0.38}\rank{1}$ & $0.32\rank{4}$ & $0.35\rank{2}$ &%BPIC2012
  $\winner{0.60}\rank{1}$ & $0.34\rank{2}$ & $0.26\rank{4}$ & $0.33\rank{2}$ \\%BPIC17
& Runtime &
   $\winner{2.76}\rank{1}$ & $4.41\rank{4}$ & $\winner{3.10}\rank{1}$ & $\winner{3.23}\rank{1}$ &%SYNTH
   $\winner{154}\rank{1}$ & $\winner{211}\rank{1}$ & $\winner{133}\rank{1}$ & $\winner{191}\rank{1}$ &%BPIC2012
   $\winner{29.1}\rank{1}$ & $81.6\rank{3}$ & $54.4\rank{2}$ & $75.6\rank{3}$ \\%BPIC17
\midrule
\multirow{5}{*}{50\%} 
%& Conf &
   %$\winner{0.99}\rank{1}$ & $\winner{1.0}\rank{1}$ & $\winner{1.0}\rank{1}$ & $\winner{1.0}\rank{1}$ &%SYNTH
  %$\winner{0.97}\rank{1}$ & $\winner{1.0}\rank{1}$ & $\winner{1.0}\rank{1}$ & $\winner{1.0}\rank{1}$ &%BPIC2012
   %$\winner{0.98}\rank{1}$ & $\winner{1.0}\rank{1}$ & $\winner{1.0}\rank{1}$ & $\winner{1.0}\rank{1}$ \\%BPIC17
& Dist. &
   $0.41\rank{4}$ & $0.27\rank{2}$ & $\winner{0.22}\rank{1}$ & $0.27\rank{2}$ &%SYNTH
   $0.55\rank{4}$ & $0.20\rank{3}$ & $\winner{0.24}\rank{1}$ & $0.30\rank{2}$ &%BPIC2012
   $0.6\rank{4}$ & $0.17\rank{2}$ & $\winner{0.09}\rank{1}$ & $0.18\rank{2}$ \\%BPIC17
& Spars. &
   $2.19\rank{4}$ & $1.41\rank{2}$ & $\winner{1.17}\rank{1}$ & $1.40\rank{2}$ &%SYNTH
   $7.83\rank{4}$ & $5.55\rank{3}$ & $\winner{3.52}\rank{1}$ & $4.11\rank{2}$ &%BPIC2012
  $7.11\rank{4}$ & $2.06\rank{2}$ & $\winner{1.06}\rank{1}$ & $2.16\rank{2}$ \\%BPIC17
& Impl. &
   $7.17\rank{2}$ & $7.17\rank{2}$ & $\winner{7.01}\rank{1}$ & $7.16\rank{2}$ &%SYNTH
   $\winner{7.58}\rank{1}$ & $8.51\rank{4}$ & $\winner{7.28}\rank{1}$ & $\winner{7.13}\rank{1}$ &%BPIC2012
   $7.95\rank{4}$ & $4.86\rank{2}$ & $\winner{4.38}\rank{1}$ & $5.14\rank{2}$ \\%BPIC17
& Dive &
   $\winner{0.43}\rank{1}$ & $0.30\rank{2}$ & $0.25\rank{4}$ & $0.30\rank{2}$ &%SYNTH
   $\winner{0.36}\rank{1}$ & $0.19\rank{4}$ & $0.25\rank{2}$ & $0.27\rank{2}$ &%BPIC2012
   $\winner{0.61}\rank{1}$ & $0.18\rank{2}$ & $0.11\rank{4}$ & $0.19\rank{2}$ \\%BPIC17
& Runtime &
   $\winner{3.39}\rank{1}$ & $9.78\rank{4}$ & $4.80\rank{2}$ & $5.18\rank{2}$ &%SYNTH
   $535\rank{2}$ & $1500\rank{4}$ & $\winner{451}\rank{1}$ & $780\rank{2}$ &%BPIC2012
   $\winner{288}\rank{1}$ & $566\rank{3}$ & $\winner{261}\rank{1}$ & $463\rank{3}$ \\%BPIC17
\bottomrule
\end{tabular}}
\caption{Performance metrics across different datasets. The ranking position of each method is indicated in parentheses.}
\label{tab:diffs}
\end{table*}

Table~\ref{tab:diffs} shows the average values of each metric for each method. In parentheses, we indicate the afferent rank of each value, indicating different rankings only for statistically significant differences.~\btext{The best performing strategy is highlighted in bold. If multiple values are in bold, this suggests no statistically significant difference between the methods in terms of the respective metric (or, in other words, that there is no statistically significant difference between the absolute best of the row and the other bold values). This also directly translates to multiple methods having the same rank.
}

We show the evolution of the results with three levels of coverage $|\tasks_\varphi|/|\tasks|$: 10\%, 25\%, and 50\%.

\paragraph{Answering RQ1.} We start by considering the time performance.
Looking at the \run of the three methods, we observe that \ltlfbase and \firstheur perform similarly, with \ltlfbase being among the best performers 7 out of 9 times, while \firstheur 5 out of 9 times. \secondheur instead shows an increase in the time required for generating the counterfactuals, especially for large coverage rates.
Regarding quality, a different story emerges: \firstheur and \secondheur demonstrate a good performance, with \firstheur showing a consistent ability to generate closer, sparser, and more plausible counterfactuals w.r.t., \ltlfbaseshort. This performance is more pronounced in real datasets (BPIC2012, and especially BPIC2017), where the complexity and length of traces are challenging the capability of \ltlfbaseshort. On the contrary, in the simpler Claim Management dataset, the superiority of \firstheur becomes more apparent only for higher coverage levels, while for lower coverages, \ltlfbaseshort exhibits better counterfactual quality results. Concerning \secondheur, it excels in balancing counterfactual quality with diversity, particularly in BPIC2012, making it a viable candidate in scenarios where both aspects are critical.

A final remark on the \sat of the counterfactuals generated with \ltlfbase. Assessing \ltlfbase's \sat is methodologically complex and beyond this paper's scope, as it is not designed to ensure compliance, which may depend upon several factors, including the \LTLp formulae used. Nonetheless, it is worth noting that in our experiments \ltlfbase managed to reach a compliance ranging on average from $80\%$ to $99\%$, depending upon the coverage used. This hints that \ltlfbase can achieve good overall \sat but cannot always guarantee satisfaction of $\varphi$. 

\paragraph{Answering RQ2.}
Overall, the \run performance analysis reveals that \firstheur and \secondheur demonstrate significantly low \run across all datasets and coverage levels, enabling a quick counterfactual generation. This efficiency is particularly pronounced at higher coverage levels, where \firstheur and \secondheur maintain fast processing times, thanks to their optimised mutation processes and reduced need for extensive checks. In contrast, \mar exhibits substantially higher runtimes, especially notable at 50\% coverage, 
due to the complexity of performing the trace validation after each mutation.

In terms of quality of the generated counterfactuals, \marshort performs similarly to our methods, particularly %in smaller datasets or 
when the task coverage is lower. As complexity increases, particularly at 50\% task coverage, \firstheur and \secondheur outperform \marshort. The counterfactuals generated by \marshort tend to have higher \spars, requiring more modifications to the original trace. In contrast, \firstheur and \secondheur are better at generating counterfactuals with less \spars. \marshort shows good performance in \impl, often ranking close to or at the top. However, \firstheur and \secondheur still maintain a good level of \impl performance, indicating their ability to generate plausible counterfactuals as well, but with added efficiency and sparsity.
Finally, \marshort does not show a significant improvement in \dive, compared to our two  strategies.
Thus, \firstheur and \secondheur offer a good trade-off between \run and quality of the generated counterfactuals, which remains overall comparable to that of \marshort.

\section{Related Work}
\label{sec:relwork}

\btext{In this section, we review related work on counterfactual explanations, focusing on their use in Predictive Process Monitoring (PPM) and temporal data. We begin by examining general techniques for generating counterfactuals, distinguishing between case-based and generative approaches. We then explore how Genetic Algorithms (GAs) are employed in counterfactual generation, noting their advantages and limitations. Lastly, we dive into recent developments in counterfactual explanations for temporal process data.}

\paragraph{\btext{Counterfactual explanations in XAI}}
Counterfactual explanations identify minimal changes to alter a model's prediction~\cite{wachter2017counterfactual}.Techniques for generating counterfactuals fall into two categories: case-based, which find counterfactuals within the sample population, and generative, creating them through~\btext{optimisation-based techniques (e.g., hill-climbing algorithms)}~\cite{verma2020counterfactual}. 

Genetic Algorithms (GAs) are widely used for generating counterfactuals by optimizing a population of potential candidates through a fitness function~\cite{moc,geco}. Both single-objective and multi-objective GA solutions are available, with single-objective solutions converging faster due to lower complexity~\cite{geco} and multi-objective GAs providing multiple optimal solutions using a Pareto Front~\cite{moc}. One key benefit of GAs is their ability to maximize population diversity, yet their stochastic nature often leads to inconsistent results.~\btext{Moreover, unlike gradient-based optimisation techniques for counterfactual generation~\cite{dice}, which require the use of differentiable models to compute counterfactuals, GAs do not require access to the model’s parameters or gradient computation. As such, they are not limited to differentiable models and do not require gradient computations by construction.}

Despite the advancements in the literature that try to incorporate plausibility and causality constraints~\cite{geco}, 
current methods have limitations in ensuring the feasibility of counterfactuals. Feasibility is crucial for generating valid counterfactuals, typically enforced through restricting the data manifold, specifying constraints, or minimizing distances to training set points~\cite{datamanifoldcfs}. 

As mentioned by~\cite{priorxai}, 
no 
background knowledge injection has been explored so far for the generation of counterfactual explanations, which is a challenge also when focusing on counterfactual explanations generated with GA approaches~\cite{Zhou2023EvolutionaryAT}. 

\paragraph{\btext{Counterfactual explanations for temporal data}}
\btext{Four works so far have tackled the counterfactual explanation problem in the PPM domain~\cite{loreley,dice4el,created,caisepaper}.}

\btext{The first paper introduces LORELEY, an adaptation of the Local Rule-Based Explanations (LORE) framework~\cite{guidotti2018local}, which generates counterfactual explanations leveraging a surrogate decision tree model using a genetically generated neighbourhood of artificial data instances to be trained~\cite{loreley,guidotti2018local}. 
The prediction task the authors address is the one of multi-class outcome prediction. To ensure the generation of feasible counterfactuals, LORELEY imposes process constraints in the counterfactual generation process by using the whole prefix of activities as a single feature, encoding the whole control-flow execution as a variant of the process.}

\btext{The second work presents \textsc{dice} for Event Logs (DICE4EL)~\cite{dice4el}. DICE4EL extends one of the methods found within \textsc{dice}~\cite{dice}, specifically, the gradient-based optimisation method by adding a feasibility term to ensure that the generated counterfactuals maximise the likelihood of belonging to the training set. To do so, DICE4EL leverages a Long-Short Term Memory (LSTM)-based predictive model as it requires gradients for the counterfactual explanation search. The prediction task addressed in the paper is that of next activity prediction with a milestone-aware focus.}

\btext{The third, the most recent approach for generating counterfactual explanations for PPM, CREATED, leverages a genetic algorithm to generate candidate counterfactual sequences~\cite{created}. To ensure the feasibility of the data, the authors build a Markov Chain, where each event is a state. Then, using the transition probabilities from one state to another, they can determine how likely a counterfactual is, given the product of the states.
}

\btext{The fourth and final work looked into proposing an evaluation framework for measuring counterfactual explanations in PPM by proposing a novel metric measuring the conformance of counterfactual generation techniques~\cite{caisepaper}. As noted by the authors, no previous approaches make use of temporal background knowledge explicitly when generating counterfactual explanations. However, background knowledge can play an important role in ensuring the feasibility of the generated counterfactuals, especially from the perspective of sequences of activities, where different constraints may have a different impact on the outcome of a trace execution. The present work aims to specifically fill this gap identified in the literature.}

%%%%%%%%%%%%%%%%%%%%%%%%%%%%%%%%%%%%%%%%%%%%%%%%%%%%%%%%%%%%%%%%%%%%%%%%
\section{Conclusions}
\label{sec:conclusions}
We have introduced a novel framework for generating counterfactual traces in temporal domains, guaranteeing that they respect background knowledge captured in a suitable temporal logic. Our approach blends automata-theoretic techniques of this logic with genetic algorithms. The results of the evaluation show that the strategies we propose %to this end at once 
ensure that background formulae remain satisfied by the generated counterfactual traces, while these traces also maintain or improve general counterfactual explanation desiderata compared to state-of-the-art methods.

In the future, we aim to develop more efficient genetic operators strategies.
We also plan to extend our approach to richer temporal logics dealing not only with activities, but also with numerical data, as in \cite{FMPW23}. This appears viable given the basis provided here, in the light of the automata-theoretic characterisation of such logics \cite{FMPW23}, as well as the fact that counterfactual desiderata can be seamlessly redefined over numerical data.
%\todo{Please check!}

\section*{Acknowledgments}
This work is partially funded by the NextGenerationEU FAIR PE0000013 project MAIPM (CUP C63C22000770006), by the PRIN MIUR project PINPOINT Prot.\~2020FNEB27, and by the PNRR project FAIR - Future AI Research (PE00000013), under the NRRP MUR program funded by the NextGenerationEU.

%\clearpage
\bibliography{main}

\clearpage
\appendix

%\input{proofs}

%\clearpage

%\input{supplementary_material}

\end{document}